\documentclass[11pt]{article}

\usepackage[preprint]{acl}
\usepackage{amsmath}
\usepackage{times}
\usepackage{latexsym}
\usepackage{amsmath}
\usepackage{amssymb}
\usepackage{amsfonts}
\usepackage{xcolor}
\usepackage{amsthm}
\newtheorem{proposition}{Proposition}
\usepackage[T1]{fontenc}

\usepackage[utf8]{inputenc}

\usepackage{microtype}
\usepackage{booktabs}
\usepackage{inconsolata}

\usepackage{graphicx}

\usepackage[most]{tcolorbox}

\newtcolorbox{promptbox}[1][]{
    enhanced,
    breakable,
    colback=gray!6,
    colframe=gray!35,
    boxrule=0.5pt,
    arc=2mm,
    left=6pt,
    right=6pt,
    top=6pt,
    bottom=6pt,
    fonttitle=\bfseries\small,
    coltitle=black,
    title=#1
}

\title{Woodpecker Distillation: Weak Models Diagnose Reasoning Bugs in Strong Models}

\author{
    \textbf{Dayu Wang}$^{1,2}$ \quad
    \textbf{Jiaye Yang}$^{1}$ \quad
    \textbf{Weikang Li}$^{1}$\thanks{Project leads and corresponding authors.} \quad
    \textbf{Jiahui Liang}$^{1}$ \quad
    \textbf{Yang Li}$^{1}$ \quad
    \textbf{Deguo Xia}$^{1}$ \quad
    \textbf{Jizhou Huang}$^{1}$\footnotemark[1] \\
    $^{1}$Baidu Inc. \qquad
    $^{2}$Nanyang Technological University \\
    \texttt{dayu001@e.ntu.edu.sg} \quad
    \texttt{yamseyoung@gmail.com} \\
    \texttt{wavejkd@pku.edu.cn} \quad
    \texttt{\{liangjiahui03, liyang164, xiadeguo, huangjizhou01\}@baidu.com}
}

\begin{document}
\maketitle
\begin{abstract}

Large language models often fail on reasoning tasks despite possessing the capability to solve them. We argue that many such failures arise from localized reasoning bugs in intermediate steps rather than from global incompetence. We show that these bugs are frequently repairable: inserting a short patch generated by a weak probe model after the same strong-model reasoning prefix can redirect the trajectory toward a correct solution.

However, this corrective effect is not reliably internalized by directly fine-tuning on weak patches or repaired trajectories, suggesting that the useful signal lies not in the intervention text itself, but in how it reshapes the model's future reasoning distribution. We therefore propose \textbf{Woodpecker Distillation}, a weak-to-strong training framework that learns from contrastive local interventions. Our method contrasts successful and unsuccessful weak-model patches at the same prefix, constructs a corrective teacher distribution from their induced future token predictions, and distills this signal into the strong model.

Experiments on mathematical reasoning benchmarks show that Woodpecker Distillation consistently improves strong-model performance and outperforms direct imitation baselines. Code and datasets are available at \url{https://anonymous.4open.science/r/Woodpecker-E6BD}.

\end{abstract}
\begin{figure}[t]
    \centering
    \includegraphics[width=0.95\columnwidth]{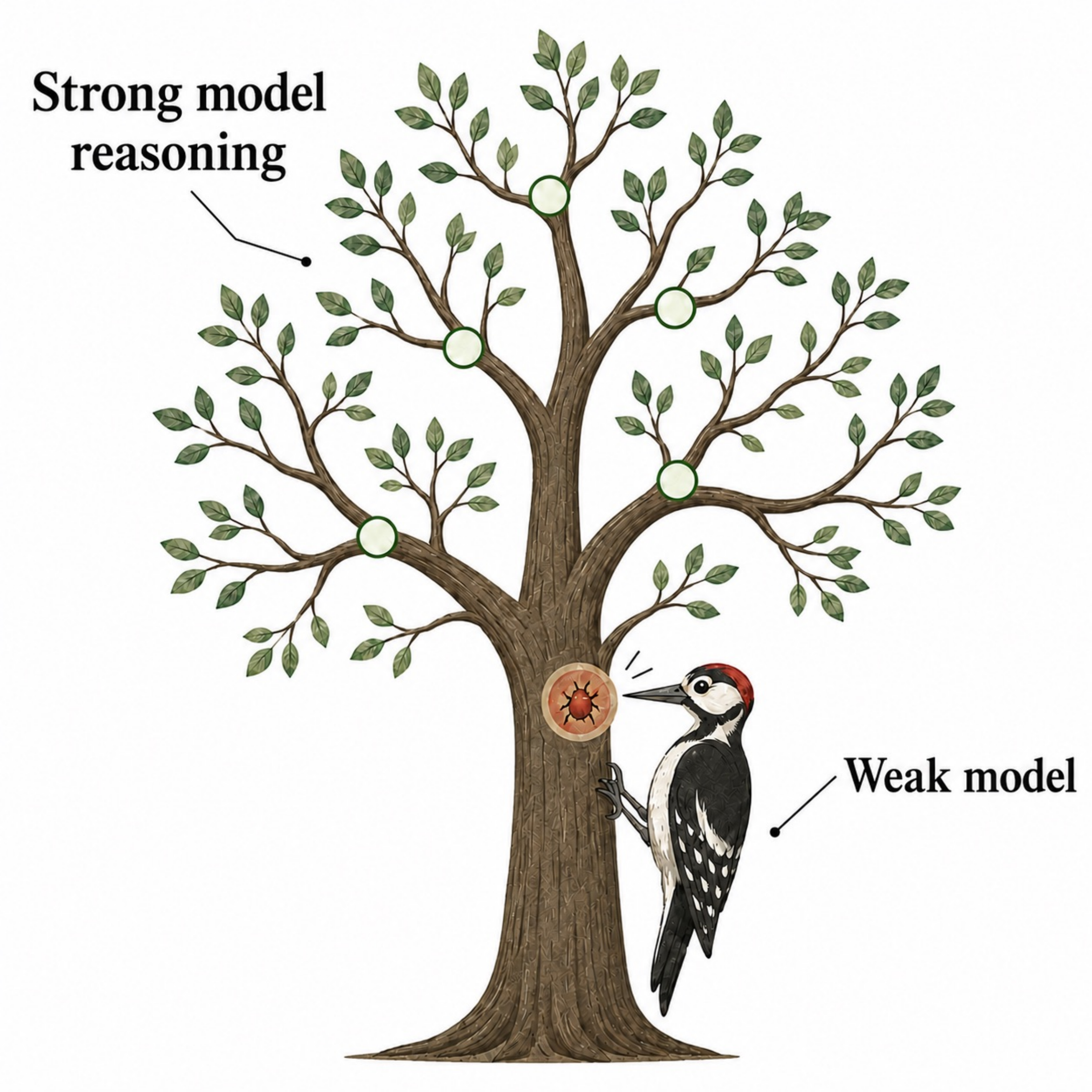}
    \caption{Woodpecker Distillation uses a weak probe model as a woodpecker to identify and repair local reasoning bugs in strong-model reasoning.}
    \label{fig:woodpecker_teaser}
    \vspace{-1em}
\end{figure}
\section{Introduction}

Recent progress in language model reasoning has largely been driven by scaling, stronger supervision, and more effective inference-time search ~\cite{kaplan2020scalinglawsneurallanguage,ouyang2022training,chung2024scaling,yao2023tree}. Yet even strong models still make brittle errors on problems they appear capable of solving ~\cite{hendrycks2021measuring,cobbe2021training}. In this work, we focus on a simple empirical phenomenon: for some failed reasoning traces, inserting a short local guidance at an intermediate prefix can make the model's subsequent continuation arrive at a correct solution. We use the term \emph{reasoning bug} to describe a recoverable failure mode in which a model's current reasoning trajectory leads to an incorrect answer, but a short local guidance inserted at an intermediate prefix can redirect its subsequent continuation toward a correct solution. This definition is operational: it does not require identifying a single erroneous token or step, but only tests whether the failure can be repaired by a local intervention.

In this paper, we ask whether a weak probe model can help expose and repair such reasoning bugs in a stronger model ~\cite{burns2023weak,kenton2024scalable}. Given an intermediate reasoning prefix from a strong model, we insert a short reasoning patch generated by a weak probe model and then let the strong model continue. We find that such patches can frequently turn otherwise incorrect continuations into correct solutions ~\cite{pan2023automatically,huang2024large}. This observation suggests that weak models may provide useful local corrective signals even when they are not stronger end-to-end problem solvers ~\cite{zhou2024weak}. In this view, the weak model is not used as a teacher that supplies complete solutions, but as a diagnostic probe that reveals which local interventions can improve the strong model's continuation.

However, directly training on these repaired trajectories is not straightforward ~\cite{ross2011reduction,bengio2015scheduled}. A natural approach is to fine-tune the strong model to imitate the probe patch itself, or to imitate the full trajectory obtained after inserting the patch ~\cite{zelikman2022star,hsieh2023distilling,kim2016sequence}. We find that these imitation-based strategies are substantially less effective than the intervention that produced the repair. This gap indicates that the useful signal is not simply the surface text of the patch, but how the patch changes the strong model's future reasoning distribution after the same prefix ~\cite{hinton2015distilling,rusu2015policy,lamb2016professor}. Internalizing this effect requires a training signal that captures the downstream consequence of a local intervention, rather than merely copying the intervention text ~\cite{ross2011reduction,de2019causal,chang2015learning}.

We therefore propose \textbf{Woodpecker Distillation}, a weak-to-strong training framework that converts local weak probe model interventions into corrective supervision. For a given strong-model prefix, we sample multiple weak probe model reasoning patches, evaluate whether each patch helps or harms the final outcome, and divide them into positive and negative intervention sets ~\cite{christiano2017deep}. We then compare the future token distributions induced by these two sets and construct a contrastive teacher distribution that emphasizes continuations associated with successful repair. The strong model is trained to match this teacher, thereby learning the corrective effect of the intervention rather than directly imitating its wording.

Our work makes three contributions. First, we identify an empirical phenomenon in weak-to-strong reasoning: some strong-model failures are recoverable reasoning bugs that can be exposed by short weak probe model interventions. Second, we show that these bugs can often be repaired by local patches, but that direct imitation of patches or patched trajectories does not reliably internalize the repair. Third, we introduce Woodpecker Distillation, a contrastive distillation method that turns successful and unsuccessful local interventions into token-level corrective supervision and improves strong-model reasoning.

\section{Related Work}
\paragraph{Weak-to-strong generalization and supervision.}
A growing literature studies how weaker supervision sources can improve stronger models, including weak-to-strong generalization, bootstrapping, and scalable oversight~\cite{lang2024theoretical,somerstep2024statistical,huang2023large,bowman2022measuring}. Most prior work focuses on transferring labels, preferences, or global solution traces from weaker sources to stronger models~\cite{ratner2016data,bai2022constitutional,ho2023large,fu2023specializing}. In contrast, our setting is neither standard weak-label supervision nor direct trace imitation: we use a weak probe model as a \emph{diagnostic probe} that perturbs a strong model's intermediate reasoning and reveals local failure modes~\cite{meng2022locating,ghandeharioun2024patchscopes}.

\paragraph{Process supervision and reasoning improvement.}
Process supervision methods improve reasoning by supervising intermediate steps, reward models, or verifier-guided search~\cite{lai2024step,setlur2025rewarding,zhang2024rest}. These approaches typically assume access to gold intermediate annotations, strong verifiers, or explicit stepwise correctness signals~\cite{zheng2025processbench}. Our method instead uses \emph{contrastive interventions}: we do not supervise which intermediate step is correct in isolation, but infer its usefulness from whether a local patch redirects the downstream reasoning trajectory toward a correct answer.

\paragraph{Self-distillation and preference-based training.}
Self-distillation and preference-style optimization train models from improved targets or pairwise comparisons~\cite{furlanello2018born,gu2024minillm,yuan2024self}. Our method is related in spirit, but differs in granularity and signal construction. Rather than distilling from a stronger teacher or optimizing over full trajectories, Woodpecker Distillation builds a \emph{localized distributional teacher} by contrasting successful and unsuccessful patches at the same reasoning prefix~\cite{yuan2023rrhf,zhao2023slic}. This lets us train on the future impact of local reasoning moves rather than on entire outputs.

\section{Motivation: Local Bugs in Strong-Model Reasoning}

We first examine whether a small local guidance can change a failed continuation. Given a problem $x$, a strong model $A$ generates a reasoning trajectory $y$, and we take an intermediate prefix $c = y_{<t}$. We instantiate $A$ as Qwen3-4B-Instruct-2507~\cite{qwen3technicalreport} and $B$ as Gemma-3-4B-IT~\cite{gemma_2025}, and evaluate on AIME 2024, AIME 2025, MATH-500, Olympiad, and Omni-Hard, where Omni-Hard contains Omni-MATH problems with difficulty greater than 7~\cite{aime24,aime25,hendrycks2021measuring,he2024olympiadbench,gao2024omnimathuniversalolympiadlevel}. We only evaluate problems on which model $A$ fails under greedy decoding. We compare two continuations from the same prefix: direct continuation from $(x,c)$, and patched continuation from $(x,c,b)$, where $b$ is a short guidance patch generated by the weak probe model $B$. Table~\ref{tab:direct_vs_aba} shows that this A--B--A intervention improves Pass@16, with larger gains on harder benchmarks such as AIME 2025 and Omni-Hard. We call this recoverable failure mode a \emph{reasoning bug}: the original continuation fails, but a short local guidance can redirect it toward a correct answer.


\begin{table}[h]
\centering
\small
\begin{tabular}{lccc}
\toprule
\textbf{Metric} & \textbf{Direct} & \textbf{A--B--A} & \textbf{$\Delta$} \\
\midrule
Overall Pass@16 & 24.21 & \textbf{26.16} & +1.95 \\
AIME 2025 Pass@16 & 57.14 & \textbf{61.90} & +4.76 \\
Omni-Hard Pass@16 & 20.62 & \textbf{23.70} & +3.08 \\
\bottomrule
\end{tabular}
\caption{
Comparison between direct continuation and A--B--A patched continuation under the same sampling budget. All numbers are percentages, and $\Delta$ denotes absolute percentage-point improvement.
}
\label{tab:direct_vs_aba}
\end{table}

This phenomenon suggests that weak probe model patches can reveal useful local repair signals. The weak probe model need not solve the full problem; it only needs to provide a short intervention that changes the strong model's subsequent continuation. 

However, directly imitating these repairs is unreliable. Fine-tuning on the weak-model patch learns only what the weak probe model says, while fine-tuning on the full patched trajectory may dilute the local repair signal among many downstream continuation tokens. Direct likelihood training can therefore overfit to the surface form of repaired examples or perturb the strong model's original reasoning distribution, rather than internalize why the local intervention was helpful.

These observations motivate Woodpecker Distillation: instead of copying weak-model patches or repaired trajectories, we train the strong model on how successful and unsuccessful local interventions change its future token distribution after the same prefix.
\begin{figure*}[t]
    \centering
    \includegraphics[width=\textwidth]{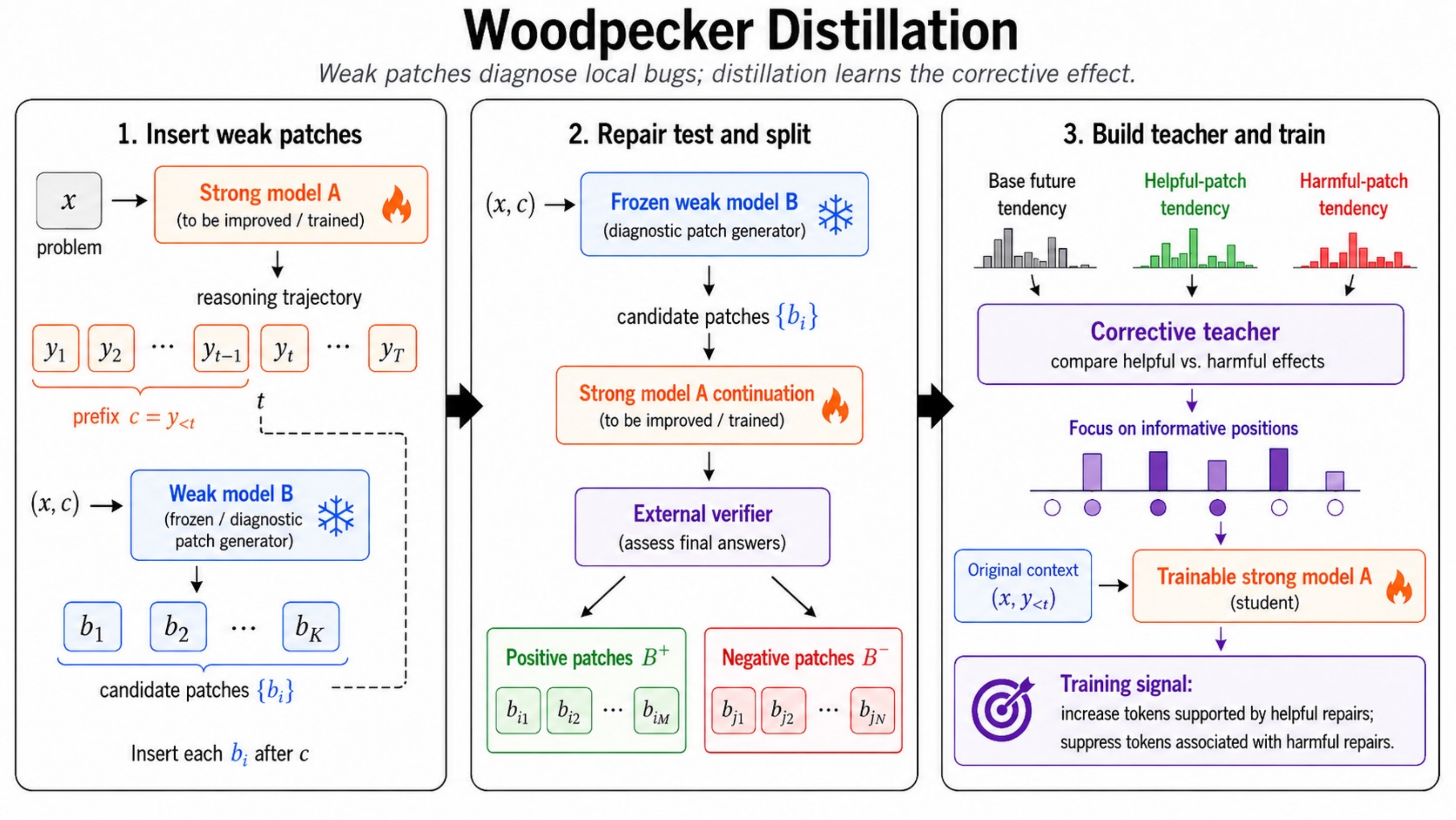}
    \caption{Overview of Woodpecker Distillation.
}
    \label{fig:method}
\end{figure*}

\section{Method}

\subsection{Overview}

Given a strong model $A$, we maintain two copies: a trainable model $A_\theta$ and a frozen reference model $A_{\mathrm{ref}}$ used to construct soft training targets. We also use a weak probe model $B$. Woodpecker Distillation is a weak-guided self-distillation framework: $B$ proposes local patches, $A_{\mathrm{ref}}$ compares their induced future token distributions, and $A_\theta$ learns the resulting corrective distribution on the original, unpatched context.

In this way, the weak probe model is not a teacher to imitate. It only provides interventions that reveal how the strong model's continuation changes. The supervision is produced by contrasting successful and unsuccessful patch-induced distributions under $A_{\mathrm{ref}}$, and the trainable model internalizes this effect without relying on weak-model patches during training or inference. Figure~\ref{fig:method} provides an overview of the
Woodpecker Distillation pipeline.

\subsection{Step 1: Insert Weak Patches}

For each problem $x$, the strong model first samples a direct reasoning trajectory
\[
y=(y_1,\ldots,y_T).
\]
We select one or more candidate insertion positions $s$ in this trajectory and denote the corresponding prefix by
\[
c_s = y_{<s}.
\]
Conditioned on the same problem and prefix, the weak probe model generates $K$ short candidate patches:
\[
\{b_i\}_{i=1}^{K} \sim B(\cdot \mid x,c_s).
\]
Each patch is treated as optional local guidance rather than a target to imitate.

\subsection{Step 2: Repair Test and Split}

We test each candidate patch by inserting it after the prefix $c_s$ and letting the strong model continue from the patched context. The completed solution is judged by an external verifier. According to the final correctness, patches at position $s$ are divided into successful and unsuccessful sets:
\[
\mathcal{B}^{+}_s
=
\{b_i:\mathrm{Verify}(A(x,c_s,b_i))=1\},
\]
\[
\mathcal{B}^{-}_s
=
\{b_i:\mathrm{Verify}(A(x,c_s,b_i))=0\}.
\]
We keep an insertion position for training only when both $\mathcal{B}^{+}_s$ and $\mathcal{B}^{-}_s$ are non-empty. Positions where all patches succeed or all patches fail are discarded, since they do not provide a contrastive intervention signal.

\subsection{Step 3: Build Teacher and Train}

For each retained position $s$, we construct supervision on a short future window of the direct trajectory:
\[
W_s=\{s,\ldots,s+L-1\}.
\]
The direct trajectory is used only to anchor local token positions; the target is not to imitate the original token $y_\tau$. Instead, we compare how successful and unsuccessful patches change the reference model's next-token distribution.

Let $h(x,c_s,b)$ denote the patched context that presents $b$ as local guidance after prefix $c_s$. The prompt is designed to avoid forcing the strong model to imitate the weak-model patch: the model is explicitly instructed to use the candidate only when it is mathematically helpful.

\begin{center}
\setlength{\fboxsep}{8pt}
\setlength{\fboxrule}{0.4pt}
\fcolorbox{black!25}{gray!6}{%
\begin{minipage}{0.92\linewidth}
\small

\noindent\textbf{Patched Prompt Template}
\vspace{0.35em}

\hrule
\vspace{0.65em}

\noindent\textbf{Problem}
\vspace{0.2em}

\noindent
\begingroup
\setlength{\fboxsep}{5pt}
\colorbox{white}{%
\begin{minipage}{0.96\linewidth}
\ttfamily\footnotesize
\{x\}
\end{minipage}
}
\endgroup

\vspace{0.55em}

\noindent\textbf{Current reasoning prefix}
\vspace{0.2em}

\noindent
\begingroup
\setlength{\fboxsep}{5pt}
\colorbox{white}{%
\begin{minipage}{0.96\linewidth}
\ttfamily\footnotesize
\{c\_s\}
\end{minipage}
}
\endgroup

\vspace{0.55em}

\noindent\textbf{Candidate local reasoning turn}
\vspace{0.2em}

\noindent
\begingroup
\setlength{\fboxsep}{5pt}
\colorbox{white}{%
\begin{minipage}{0.96\linewidth}
\ttfamily\footnotesize
\{b\}
\end{minipage}
}
\endgroup

\vspace{0.65em}

\noindent\emph{Use the candidate only if it is mathematically helpful.}

\vspace{0.25em}

\noindent Continue solving the problem step by step.

\end{minipage}
}
\end{center}

For each patch $b$ and position $\tau\in W_s$, we evaluate the frozen reference model:
\[
p^{b}_{\tau}(\cdot)
=
p_{\mathrm{ref}}(\cdot \mid h(x,c_s,b), y_{s:\tau}).
\]
We then aggregate the patch-induced distributions over successful and unsuccessful patches:
\[
p^{+}_{\tau}
=
\frac{1}{|\mathcal{B}^{+}_s|}
\sum_{b\in \mathcal{B}^{+}_s}
p^{b}_{\tau},
\qquad
p^{-}_{\tau}
=
\frac{1}{|\mathcal{B}^{-}_s|}
\sum_{b\in \mathcal{B}^{-}_s}
p^{b}_{\tau}.
\]
Here $p^{+}_{\tau}$ captures future token predictions associated with helpful interventions, while $p^{-}_{\tau}$ captures those associated with unhelpful interventions.

We construct a corrective soft teacher relative to the reference model's original prediction on the unpatched trajectory:
\[
p^{0}_{\tau}
=
p_{\mathrm{ref}}(\cdot \mid x,y_{<\tau}).
\]
For each vocabulary token $v$, define
\[
\Delta_{\tau}(v)
=
\mathrm{clip}
\left(
\log p^{+}_{\tau}(v)-\log p^{-}_{\tau}(v),
-\delta,\delta
\right),
\]
and
\[
q_{\tau}(v)
=
\frac{
p^{0}_{\tau}(v)
\exp\left(\eta \Delta_{\tau}(v)\right)
}{
\sum_{v'}
p^{0}_{\tau}(v')
\exp\left(\eta \Delta_{\tau}(v')\right)
}.
\]
The parameter $\eta$ controls the correction strength, and $\delta$ prevents extreme log-ratio shifts.

\begin{table*}[t]
\centering
\small
\setlength{\tabcolsep}{5pt}
\begin{tabular}{lcccccc}
\toprule
\textbf{Method} 
& \textsc{AIME24} 
& \textsc{AIME25} 
& \textsc{MATH-500} 
& \textsc{Olympiad} 
& \textsc{Omni-Hard} 
& \textbf{Avg.} \\
\midrule
Base strong model 
& 53.3 
& 30.0 
& 85.6 
& 56.5 
& 20.2 
& 53.4 \\

Weak probe model
& 6.7 
& 10.0 
& 72.6 
& 34.6 
& 9.1 
& 36.8 \\

Self-rejection SFT
& 50.0
& 40.0 
& 85.2 
& 57.9 
& 20.4 
& 53.8 \\

Woodpecker Distillation 
& \textbf{63.3} 
& \textbf{43.3}
& \textbf{86.4} 
& \textbf{58.3} 
& \textbf{20.6} 
& \textbf{54.8} \\
\bottomrule
\end{tabular}
\caption{
Greedy decoding accuracy on mathematical reasoning benchmarks. The average is weighted by the number of problems in each dataset.
}
\label{tab:main-greedy}
\end{table*}
To focus training on informative positions, we use a Jensen--Shannon gate:
\[
m_{\tau}
=
\mathrm{clip}
\left(
\frac{
\mathrm{JS}(p^{+}_{\tau},p^{-}_{\tau})
}{
\tau_{\mathrm{js}}
},
0,1
\right).
\]
When successful and unsuccessful patches induce similar distributions, the position receives little weight.

Finally, the trainable model is optimized on the original, unpatched context:
\[
\mathcal{L}_{\mathrm{WD}}
=
\sum_{s}
\sum_{\tau\in W_s}
m_{\tau}\,
\mathrm{CE}
\left(
q_{\tau},
\pi_{\theta}(\cdot \mid x,y_{<\tau})
\right).
\]
Thus, $A_\theta$ never observes the weak-model patch during optimization. The patch is used only to construct the soft teacher, while the student learns to reproduce the corrective distributional effect from the original context alone.

\section{Experiments}
\label{sec:experiments}

\subsection{Experimental Setup}
\paragraph{Models.}
We instantiate the strong model $A$ with Qwen3-4B-Instruct-2507 and the weak probe model $B$ with Gemma-3-4B-IT. The trainable copy $A_\theta$ is initialized from Qwen3-4B-Instruct-2507, while $A_{\mathrm{ref}}$ is kept frozen to construct soft training targets.

\paragraph{Datasets.}
We evaluate all models on five mathematical reasoning benchmarks: MATH-500, Olympiad-test, Omni-Hard, AIME 2024, and AIME 2025. For each generated response, we extract the final answer using regular-expression rules and verify correctness with the \texttt{math-verify} toolkit. To ensure fair comparison, all methods are evaluated with the same system prompt, conversation template, and context-length setting.

For training, we sample problems from Skywork-OR1-RL-Data~\citep{he2025skywork,skywork-or1-2025}. We remove all examples that are identical to, or substantially overlap with, the evaluation benchmarks before training. Additional details on training and experimental settings are provided in Appendix~\ref{app:exp_details}.

\subsection{Baselines}
We compare Woodpecker Distillation with the following baselines.

\paragraph{Base strong model.}
This is the original strong model $A$ without any additional training. It measures the starting reasoning ability of the target model.

\paragraph{Weak probe model.}
This is the original Weak probe model $B$.

\paragraph{Self-rejection sampling fine-tuning.}
Using the same training dataset and sampling budget as our method, this baseline filters verifier-correct trajectories from strong-model self-sampling and fine-tunes on the resulting successful solutions.

\paragraph{Woodpecker Distillation.}
Our method does not imitate the weak-model patch or the patched trajectory. Instead, it contrasts successful and unsuccessful interventions and distills their induced future token distributions into the strong model.

\subsection{Main Results: Greedy Evaluation}
\label{sec:main-results}

Table~\ref{tab:main-greedy} reports greedy decoding accuracy across five mathematical reasoning benchmarks. Woodpecker Distillation improves the base strong model on all datasets, increasing the average accuracy from 53.4 to 54.8. The gains are most pronounced on AIME 2024 and AIME 2025, where accuracy improves by 10.0 and 13.3 percentage points, respectively.

Woodpecker Distillation also outperforms self-rejection SFT under the same training data and sampling budget. This comparison suggests that the improvement is not simply due to filtering verifier-correct trajectories from the strong model. Instead, using verifier outcomes to contrast successful and unsuccessful local interventions provides a more targeted supervision signal than training on complete self-generated correct solutions.

\subsection{Ablation Studies}
\label{sec:ablations}

\begin{table*}[t]
\centering
\small
\setlength{\tabcolsep}{5pt}
\begin{tabular}{lcccccc}
\toprule
\textbf{Method} 
& \textsc{AIME24} 
& \textsc{AIME25} 
& \textsc{MATH-500} 
& \textsc{Olympiad} 
& \textsc{Omni-Hard} 
& \textbf{Avg.} \\
\midrule
Woodpecker Distillation 
& \textbf{63.3} 
& \textbf{43.3}
& \textbf{86.4} 
& \textbf{58.3} 
& \textbf{20.6} 
& \textbf{54.8} \\

\quad w/ same-model probe ($A=B$)
& 56.7
& \textbf{43.3}
& 85.8
& 55.8
& 20.2
& 53.4 \\

\quad w/o negative interventions
& 3.3
& 6.7
& 46.8
& 16.6
& 7.8
& 22.1 \\

\quad w/o JS gate
& 56.7
& 36.7
& 84.8
& 56.1
& 19.9
& 53.0 \\
\bottomrule
\end{tabular}
\caption{
Ablation results on mathematical reasoning benchmarks under greedy decoding.
}
\label{tab:ablation}
\end{table*}

\begin{table*}[t]
\centering
\small
\setlength{\tabcolsep}{5pt}
\begin{tabular}{llcccccc}
\toprule
\textbf{Seed} 
& \textbf{Probe Model}
& \textsc{AIME24} 
& \textsc{AIME25} 
& \textsc{MATH-500} 
& \textsc{Olympiad} 
& \textsc{Omni-Hard} 
& \textbf{Avg.} \\
\midrule
2026
& Gemma-3-4B-IT
& 60.0 
& \textbf{50.0}
& 85.6 
& \textbf{58.8} 
& 20.4 
& 54.7 \\

42
& Gemma-3-4B-IT
& \textbf{63.3}
& 43.3
& \textbf{86.4}
& 58.3
& 20.6
& 54.8 \\

42
& Mistral-7B-Instruct-v0.3
& \textbf{63.3}
& 46.7
& 85.4
& 58.2
& \textbf{21.0}
& 54.6 \\
\bottomrule
\end{tabular}
\caption{
Robustness of Woodpecker Distillation to the choice of probe model.
The strong model is fixed as Qwen3-4B-Instruct-2507, while the auxiliary probe model is varied between Gemma-3-4B-IT and Mistral-7B-Instruct-v0.3.
All results are greedy decoding pass@1 accuracy in percentage points over the same 1,763-problem evaluation suite.
}
\label{tab:probe-robustness}
\end{table*}

We conduct ablations to evaluate the main design choices of Woodpecker Distillation. Table~\ref{tab:ablation} shows that each component contributes to the final performance.

\paragraph{Same-model probe.}
We first replace the weak probe model $B$ with the strong model itself, i.e., $A=B$. This variant removes the asymmetric weak-to-strong probing setup while keeping the rest of the method unchanged. Its average accuracy drops from 54.8 to 53.4, suggesting that weak-model patches provide useful intervention diversity beyond self-generated perturbations from the strong model.

\paragraph{Removing negative interventions.}
We next remove the negative intervention distribution and construct the teacher only from successful patches. This variant collapses substantially, with average accuracy dropping to 22.1. Without the negative distribution, the teacher lacks a reference direction for suppressing harmful intervention-induced shifts, which can over-amplify patch-specific artifacts.

\paragraph{Removing JS gating.}
Finally, we remove the Jensen--Shannon gate by setting $m_\tau=1$ for all positions. Performance decreases from 54.8 to 53.0, showing that not all token positions provide equally reliable repair signals. Weighting positions by the divergence between positive and negative intervention distributions helps focus training on locations where the contrast is informative.

Overall, these ablations support the three core ingredients of Woodpecker Distillation: weak-model probing provides useful local diversity, positive-negative contrast defines the corrective direction, and JS gating filters noisy or uninformative positions.

\subsection{Robustness Analysis}
\label{sec:robustness}

Table~\ref{tab:probe-robustness} evaluates the robustness of Woodpecker Distillation under different random seeds and probe models. The results remain stable across two Gemma-3-4B-IT runs and when replacing the probe with Mistral-7B-Instruct-v0.3, providing initial evidence that the method is not tied to a single weak probe model instance.

\subsection{Cost Analysis}
\begin{figure}[h]
    \centering
    \includegraphics[width=0.75\linewidth]{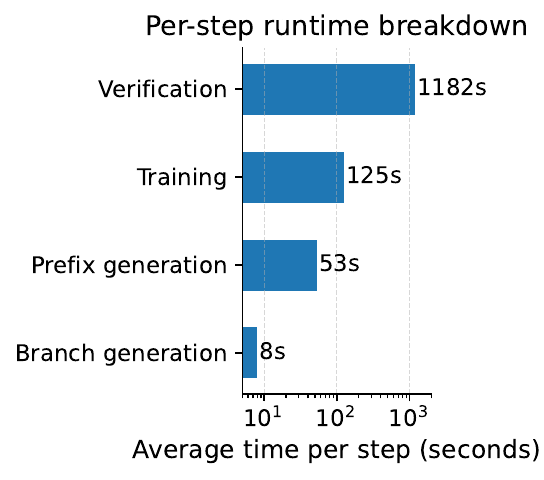}
    \caption{Average per-step runtime breakdown of Woodpecker Distillation, measured on H800 GPUs.}
    \label{fig:runtime-breakdown}
\end{figure}
\label{sec:cost}

We analyze the per-step runtime of Woodpecker Distillation in Figure~\ref{fig:runtime-breakdown}. Each training step consists of four components: generating the strong-model prefix, generating weak-model branches, verifying patched continuations, and applying gradient updates. The runtime is dominated by verification, which takes 1182 seconds on average and accounts for 86.4\% of the total per-step time. In contrast, weak-branch generation takes only 8 seconds, contributing less than 1\% of the total cost.

This breakdown shows that the main overhead of Woodpecker Distillation does not come from the weak probe model itself. Since the weak probe model only generates short local patches, its cost is small compared with evaluating completed continuations. Training takes 125 seconds per step, indicating that optimization is also not the primary bottleneck. Therefore, the efficiency of the method is mainly determined by the cost of verifier-based evaluation, rather than by weak-model patch generation.


\section{Analysis}
\label{sec:analysis}

Woodpecker Distillation constructs a contrastive soft teacher from
local interventions. In this section, we provide a conservative
interpretation of this teacher. Our goal is not to prove that the
resulting token updates causally guarantee better final answers.
Instead, we show that the teacher can be viewed as a KL-regularized
distributional update toward tokens that are more associated with
successful local interventions than unsuccessful ones.

\subsection{Patch-Conditioned Intervention Distributions}

Fix a problem $x$ and a strong-model prefix $c$. A weak probe model
samples a local patch
\[
b \sim B(\cdot \mid x,c).
\]
After inserting $b$ after the prefix, the strong model continues from
the patched context. Let
\[
S(b) \in \{0,1\}
\]
denote the verifier outcome of this patched continuation, where
$S(b)=1$ means that the patched continuation reaches a correct final
answer. This outcome is used only to partition patches; it is not
treated as a token-level label.

The successful and unsuccessful patch distributions are defined as
\[
\mu^+(b\mid x,c)
\propto
B(b\mid x,c)\mathbf{1}\{S(b)=1\},
\]
and
\[
\mu^-(b\mid x,c)
\propto
B(b\mid x,c)\mathbf{1}\{S(b)=0\}.
\]
In practice, these distributions are estimated by averaging over the
finite sets of successful and unsuccessful sampled patches.

For a future token position $\tau$, we evaluate the frozen reference
model under each patched context and aggregate the resulting
next-token distributions:
\[
p^+_\tau(v)
=
\mathbb{E}_{b\sim \mu^+}
\left[
p_{\mathrm{ref}}(v \mid h(x,c,b), y_{s:\tau})
\right],
\]
\[
p^-_\tau(v)
=
\mathbb{E}_{b\sim \mu^-}
\left[
p_{\mathrm{ref}}(v \mid h(x,c,b), y_{s:\tau})
\right].
\]
Here $h(x,c,b)$ denotes the context obtained by inserting patch $b$
after prefix $c$, and $y_{s:\tau}$ is the original trajectory segment
used as a common teacher-forced alignment path. This alignment avoids
comparing next-token distributions at mismatched free-running
continuations, but it should be viewed as a local approximation rather
than a full trajectory-level estimate.

We define the contrastive repair score
\[
r_\tau(v)
=
\log p^+_\tau(v)-\log p^-_\tau(v).
\]
This score should be interpreted as a contrastive proxy for local
repair value: it measures how strongly token $v$ is associated with
distributions induced by successful interventions relative to
unsuccessful ones under the same prefix.

\begin{proposition}[Positive-negative contrast]
If $r_\tau(v)>0$, then token $v$ receives higher probability under the
successful-intervention distribution than under the
unsuccessful-intervention distribution. If $r_\tau(v)<0$, then token
$v$ is relatively more associated with unsuccessful interventions.
\end{proposition}

\noindent
This proposition is only an associational statement. It does not claim
that increasing the probability of $v$ alone will necessarily make the
final answer correct. Rather, it motivates using $r_\tau(v)$ as a local
directional signal for aligning the model's prediction with the
distributional effect of successful interventions.

\subsection{KL-Regularized Contrastive Teacher}

Let $p^0_\tau$ denote the frozen reference model's original prediction
on the unpatched context:
\[
p^0_\tau(v)
=
p_{\mathrm{ref}}(v \mid x,y_{<\tau}).
\]
We clip the contrastive score to reduce the effect of unstable or
extreme log-ratios:
\[
\tilde r_\tau(v)
=
\mathrm{clip}
\left(
\log p^+_\tau(v)-\log p^-_\tau(v),
-\delta,\delta
\right).
\]

For each token position $\tau$, consider the following optimization
over distributions $q$:
\[
q^*_\tau
=
\arg\max_q
\left[
\mathbb{E}_{v\sim q}\left[\tilde r_\tau(v)\right]
-
\frac{1}{\eta}
\mathrm{KL}(q\|p^0_\tau)
\right],
\]
where $\eta>0$ controls the strength of the update. The first term
encourages probability mass on tokens associated with successful
interventions, while the KL term keeps the teacher close to the
original strong-model prediction.

\begin{proposition}[KL-regularized tilted teacher]
The solution to the above optimization is
\[
q^*_\tau(v)
=
\frac{
p^0_\tau(v)\exp\left(\eta \tilde r_\tau(v)\right)
}{
\sum_{v'}
p^0_\tau(v')\exp\left(\eta \tilde r_\tau(v')\right)
}.
\]
\end{proposition}

\begin{proof}
The Lagrangian for the constrained optimization over probability
distributions is
\[
\begin{aligned}
\mathcal{J}(q,\lambda)
=&
\sum_v q(v)\tilde r_\tau(v)
-
\frac{1}{\eta}
\sum_v q(v)\log\frac{q(v)}{p^0_\tau(v)}
\\
&\quad
+
\lambda\left(\sum_v q(v)-1\right).
\end{aligned}
\]
Taking the derivative with respect to $q(v)$ and setting it to zero
gives
\[
\tilde r_\tau(v)
-
\frac{1}{\eta}
\left(
\log\frac{q(v)}{p^0_\tau(v)}+1
\right)
+
\lambda
=
0.
\]
Therefore,
\[
q(v)
\propto
p^0_\tau(v)\exp\left(\eta \tilde r_\tau(v)\right).
\]
Normalizing over the vocabulary gives the stated form.
\end{proof}

This proposition shows that the teacher is an exponentially tilted
version of the original reference distribution. The tilt is determined
by the clipped positive-negative contrast, while the KL term prevents
the target from moving arbitrarily far from the model's original
prediction. If $p^+_\tau=p^-_\tau$, then $\tilde r_\tau(v)=0$ for all
$v$ and $q^*_\tau=p^0_\tau$, so the teacher reduces to the original
reference distribution.

\subsection{Contrast Weighting with Jensen-Shannon Divergence}

The contrastive score is useful only when successful and unsuccessful
patches induce meaningfully different future-token distributions. If
$p^+_\tau$ and $p^-_\tau$ are nearly identical, their log-ratio is
likely to provide little useful training signal. We therefore weight
each position by
\[
m_\tau
=
\mathrm{clip}
\left(
\frac{\mathrm{JS}(p^+_\tau,p^-_\tau)}
{\tau_{\mathrm{js}}},
0,1
\right),
\]
where $\mathrm{JS}(\cdot,\cdot)$ is the Jensen-Shannon divergence.

\begin{proposition}[Contrast-based weighting]
The weight $m_\tau$ is larger when the successful and unsuccessful
intervention distributions are more separated, and becomes small when
the two distributions are similar.
\end{proposition}

\noindent
We use this weight as a heuristic measure of contrast strength, not as
a formal statistical confidence estimate. Since $p^+_\tau$ and
$p^-_\tau$ are finite-sample estimates, large divergence can still be
noisy when the number of sampled patches is small. The clipping of
$\tilde r_\tau$ and the weighting by $m_\tau$ are therefore practical
stabilization mechanisms rather than guarantees of causal correctness.

The final training objective is
\[
\mathcal{L}_{\mathrm{WD}}
=
\sum_{\tau}
m_\tau\,
\mathrm{CE}
\left(
q^*_\tau,
\pi_\theta(\cdot\mid x,y_{<\tau})
\right).
\]
The trainable model is optimized on the original unpatched context.
Thus, weak-model patches are used only to construct the contrastive soft
teacher; they are not directly imitated during training.

\section{Conclusion}

We introduced Woodpecker Distillation, a weak-to-strong training framework that uses weak-model patches as diagnostic probes for localized reasoning failures in strong language models. By contrasting the future token distributions induced by successful and unsuccessful local interventions, our method constructs a corrective soft teacher and distills it into the strong model on the original unpatched context. Experiments on mathematical reasoning benchmarks show that this contrastive local supervision improves greedy reasoning accuracy and outperforms direct self-rejection fine-tuning, suggesting that weak probe models can provide useful corrective signals even when they are not stronger end-to-end reasoners.

\section{Limitations}

Our work has several limitations. First, the current evaluation focuses on mathematical reasoning benchmarks, so further experiments are needed to test whether the same form of local repair transfers to broader reasoning domains, such as code generation, factual reasoning, or multi-step planning. Second, Woodpecker Distillation depends on verifier-based feedback to separate successful and unsuccessful interventions; this makes the method most directly applicable to domains where reliable automatic verification is available. Third, the main computational bottleneck is evaluating patched continuations, rather than generating weak-model patches or performing training. Reducing this verification cost is an important direction for making the method more scalable. Finally, our method targets recoverable local reasoning bugs. It may be less effective when the strong model lacks the underlying capability to solve the problem, or when failures arise from global planning errors rather than locally repairable reasoning steps.

\bibliography{custom}

\appendix

\section{Experimental Details}
\label{app:exp_details}
\paragraph{Training data processing.}
We construct the training set from the mathematical training split of Skywork-OR1-RL-Data. 
To avoid benchmark contamination, we first remove problems that are identical or highly similar to examples in the evaluation benchmarks. 
From the remaining pool, we select problems whose difficulty scores under DeepSeek-R1-Distill-Qwen-32B fall between 12 and 15. 
We then evaluate Qwen3-4B-Instruct-2507 with eight sampled responses per problem and retain 2,956 problems with partial success, i.e., problems whose Pass@8 count is greater than 0 and smaller than 8. 
To further include challenging examples that the actor cannot solve under its current sampling distribution, we additionally select 800 problems with the shortest labels from the subset with Pass@8 count equal to 0. 
This yields a final training set of 3,756 mathematical reasoning problems.
\paragraph{Implementation hyperparameters.}
Unless otherwise specified, we use Qwen3-4B-Instruct-2507 as the strong model $A$, including the trainable model $A_{\theta}$ and the frozen reference model $A_{\mathrm{ref}}$, and Gemma-3-4B-IT as the weak probe model $B$. For each problem, the strong model first generates a direct trajectory with greedy decoding. We use candidate insertion positions $s \in \{100, 200, 400\}$ and sample $K=16$ weak-model patches at each position. Each patch is generated with temperature $1.0$ and top-$p=0.95$, with a maximum length of 32 tokens, and is truncated to 16 tokens under the strong-model tokenizer before insertion. Patched continuations use a maximum generation length of 8192 tokens and are judged by final-answer extraction followed by \texttt{math-verify}. We retain an insertion position only when it contains at least one successful and one unsuccessful patch. For the contrastive teacher, we set the correction strength to $\eta=2.0$, the log-ratio clipping threshold to $\delta=4.0$, and the Jensen--Shannon gating threshold to $\tau_{\mathrm{js}}=0.002$. All evaluation runs use the same system prompt, conversation template, and context length across methods.

\paragraph{Artifact licenses and intended use.}
We use publicly released models, datasets, and evaluation tools, including Qwen3-4B-Instruct-2507, Gemma-3-4B-IT, Skywork-OR1-RL-Data, mathematical reasoning benchmarks, and math-verify. We use these artifacts only for research purposes and cite their original creators. Our released code and processed data are intended for research use only, subject to the licenses and terms of the corresponding original artifacts.
\appendix
\end{document}